%% file: main.tex
\def\year{2021}\relax
\documentclass[letterpaper]{article} %
\pdfoutput=1
\usepackage{aaai21}  %
\usepackage{times}  %
\usepackage{helvet} %
\usepackage{courier}  %
\usepackage[hyphens]{url}  %
\usepackage{graphicx} %
\usepackage{graphicx}  %
\usepackage{natbib}  %
\usepackage{caption} %
\usepackage{amsmath}
\usepackage{amsfonts}
\usepackage{amssymb}
\usepackage{amsthm}
\usepackage{booktabs}
\usepackage{mathtools,lipsum}
\usepackage{cuted}

\usepackage{enumitem}
\usepackage{graphicx}
\usepackage{subcaption}

\usepackage{pgfplots} 

\usepackage{amsfonts}       %
\usepackage{dsfont}
\usepackage{nicefrac}       %

\pgfplotsset{compat=1.14}

\input{defs.tex}

\newtheorem{theorem}{Theorem}

\newtheorem{lemma}{Lemma}

\newtheorem{example}{Example}
\newcommand{\comment}[1]{}
\title{An Extension of Fano's Inequality for Characterizing  Model Susceptibility to Membership Inference Attacks}
\author{
    Sumit Kumar Jha \textsuperscript{\rm 1}, Susmit Jha \textsuperscript{\rm 2}, Rickard Ewetz \textsuperscript{\rm 3}, Sunny Raj \textsuperscript{\rm 4}, Alvaro Velasquez \textsuperscript{\rm 5}\\ Laura L. Pullum \textsuperscript{\rm 6}, Ananthram Swami\textsuperscript{\rm 7}\\
}
\affiliations{
    \textsuperscript{\rm 1} University of Texas at San Antonio, TX, 78249\\
    \textsuperscript{\rm 2} SRI International, Menlo Park, CA, 94025\\
    \textsuperscript{\rm 3} University of Central Florida, Orlando, FL 32816\\
    \textsuperscript{\rm 4} Oakland University, Rochester, MI, 48309 \\
    \textsuperscript{\rm 5} Air Force Research Laboratory, Rome, NY, 13441 \\
    \textsuperscript{\rm 6} Oak Ridge National Laboratory, Oak Ridge, TN, 37831\\
    \textsuperscript{\rm 7} Army Research Laboratory, Adelphi, MD 20783
    
}

\begin{document}

\maketitle

\begin{abstract}
Deep neural networks have been  shown to be vulnerable to membership inference attacks wherein the attacker aims to detect whether specific input data were used to train the model. These attacks can potentially leak private or proprietary data. %
We present a new extension of Fano's inequality and employ it to theoretically establish that the probability of success for a membership inference attack on a deep neural network can be bounded using the mutual information between its inputs and its activations and/or outputs. This enables the use of mutual information to measure the susceptibility of a DNN model to membership inference attacks. %
In our empirical evaluation, we show that the correlation between the  mutual information and the susceptibility of the DNN model to membership inference attacks is 0.966, 0.996, and 0.955 for CIFAR-10, SVHN and GTSRB models, respectively. 

\end{abstract}

\input{secintro.tex}

\input{secsummary.tex}

\input{sectheorem.tex}

\input{secalgo.tex}

\input{secrelated.tex}

\input{secexp.tex}

\section{Conclusions and Future Work}

 Fano's inequality is a classical information theoretic result that relates the probability of an error in a channel with the conditional entropy between the input and output of a noisy channel. We present a new extension to Fano's inequality~\cite{fano1961transmission} that  establishes a bound on the success probability of a membership inference attack using mutual information between the inputs and the outputs/activations of a DNN model.  We mathematically prove that our mutual information based bound can measure a DNN model's susceptibility to any membership attack. 
 
 In our empirical evaluation, the correlation between the mutual information and the susceptibility of the DNN model to membership inference attacks is 0.966, 0.996, and 0.955 for CIFAR-10, SVHN and GTSRB, respectively. Thus, we address the challenge of making DNNs less susceptible to membership inference attacks and reduce the risk of inadvertent leak of information about training data.

Several directions for future research remain open. While this paper focuses on the use of mutual information as a susceptibility metric another interesting line of research may focus on computing $p_{\alpha}$ directly as a metric of susceptibility to membership inference attacks. Since mutual information $\mi{\inp{};\out{}}$ is the only term in the lower bound of Theorem~\ref{mainthem} that arises from the design and training of the neural network, we have chosen to focus on mutual information as a susceptibility metric. 

Because of recent advances in neural network based estimation of mutual information, our results on $\mi{\inp{};\out{}}$ as a metric can be used to create an effective regularization approach for training neural networks that are more robust against membership inference attacks.

Another interesting direction of research is a deeper understanding of the tightness of our bound based on mutual information. While we have presented experimental evidence on three different data sets to show that mutual information is a good metric for measuring model susceptibility to membership inference attacks, a theoretical investigation into the tightness of the bound may lead to deeper insights.

\newpage
\section*{Ethical and Broader Impact}

There is an emerging trend of providing DNN models to users either directly or through cloud services, where the model has been trained on proprietary or private data. The recently proposed membership inference attacks show that the user of the model can infer whether a training data was used in a model or not. 
MIA attacks  violate the expected privacy of  the individual participants contributing to the training data, and  cause unauthorized leakage of the training dataset which could be of business value or even a trade secret. For example, membership in the training data set of a model associated with a disease or addiction can reveal otherwise private information about a patient.  As yet another example, consider an anomaly detection DNN model for an engine made available to customers by the engine manufacturer, the discovery of training data employed for anomaly detection could leak crucial proprietary information.
These concerns create a hurdle to the broader adoption of DNN models. 

We address this socially important challenge in the paper. We present a way to analyze a machine learning model to understand its susceptibility to membership inference attack using mutual information between the inputs and the outputs of the model.  Our approach will make machine learning models more robust and privacy-aware, and thus, be of positive impact to society.  %
\bibliography{main}

\end{document}

%% file: defs.tex
\newcommand \nn {\mathcal{N}}

\newcommand \ds {\mathcal{D}}

\newcommand {\inp} [1] {D_{#1}}

\newcommand {\out} [1] {Y_{#1}}

\newcommand \mia {\mathcal{M}}

\newcommand \gt {X}

\newcommand {\gti} [1] {X_{#1}}

\newcommand {\ai} [1] {A_{#1}}

\newcommand{\error}{\xi}
\newcommand{\ea}{E_\alpha}

\newcommand{\mi}[1]{\mathbb{I}(#1)}
\newcommand{\en}[1]{\mathbb{H}(#1)}

\newcommand{\pa}{p_{\alpha}}

\newcommand{\binomialseries}{\log{\left( {\scriptstyle \binom{|\ds{}|}{0} + \dots + \binom{|\ds{}|}{\alpha} } \right)}}

%% file: secintro.tex
\section{Introduction}

Deep neural network (DNN) models have achieved remarkable accuracy levels on tasks such as image classification, activity recognition, speech translation, autonomous driving, and medical diagnosis. This has fueled the emergence of a market for DNN models that could be trained on proprietary or private data, and then made available to the users either directly or as a service over cloud platforms. 
Recently, it has been shown that black-box access to a DNN model can be used to detect 
whether a specific data item is a member of the training data set. Such membership inference attacks (MIA) pose a significant security and privacy risk. 

The ``Dalenius desideratum''~\cite{dwork2011firm} was first proposed in the literature on statistical disclosure control and attempts to characterize this notion of expected privacy for training data. It states that the model should reveal no more about the input to which it is applied than would have been known about this input without applying the model. Another closely related notion of privacy considers the leak in the values of sensitive protected attributes of an input by using the model's output~\cite{fredrikson2014privacy}.
But such absolute notions of privacy for all training inputs cannot be achieved by any useful model~\cite{dwork2010difficulties}. A membership inference attack using the neural network's top layer output was shown in~\cite{shokri2015privacy}, and a recent improvement, by incorporating activation and gradient output of layers, was proposed in~\cite{nasr2019comprehensive}.
Techniques such as those employing differential privacy during model training have also been shown to be not immune to privacy attacks without deterioration of the model's accuracy~\cite{rahman2018membership}. 
A useful model must preserve some information of the training data to make accurate predictions. The literature on generalization in deep learning~\cite{zhang2016understanding,neyshabur2017exploring} studies a closely related problem of understanding whether the model has memorized training data or distilled a generalized model from it. Some theories of generalization in deep learning connect it to the mutual information between the input and output of the model~\cite{shwartz2017opening,xu2017information}. 
We make the following  contributions in this paper:
\begin{itemize}
 \item Fano's inequality establishes an information theoretic relationship between the average information lost in a noisy channel and the probability of the categorization error~\cite{fano1961transmission}. We extend Fano's inequality to   establish that the probability of success for a  membership inference attack on a deep neural network can be bounded by an expression that depends on the mutual information between its inputs and its activations and/or outputs.

\item Inspired by our theoretical results, we use the mutual information between the input and the outputs/activations of a DNN model as a metric for computing its susceptibility to membership inference attacks (MIA). Our  evaluation over a set of deep learning benchmarks and membership attack~\cite{shokri2015privacy, nasr2019comprehensive} methods demonstrates that mutual information strongly correlates with the success probability of membership inference attacks. Our experimental results show that the correlation between the  mutual information and MIA susceptibility  is 0.966, 0.996, and 0.955 for CIFAR-10, SVHN and GTSRB data sets.

\end{itemize}

%% file: secsummary.tex
\begin{figure*}[htbp]
   \centering
   \includegraphics[width=15.2cm]{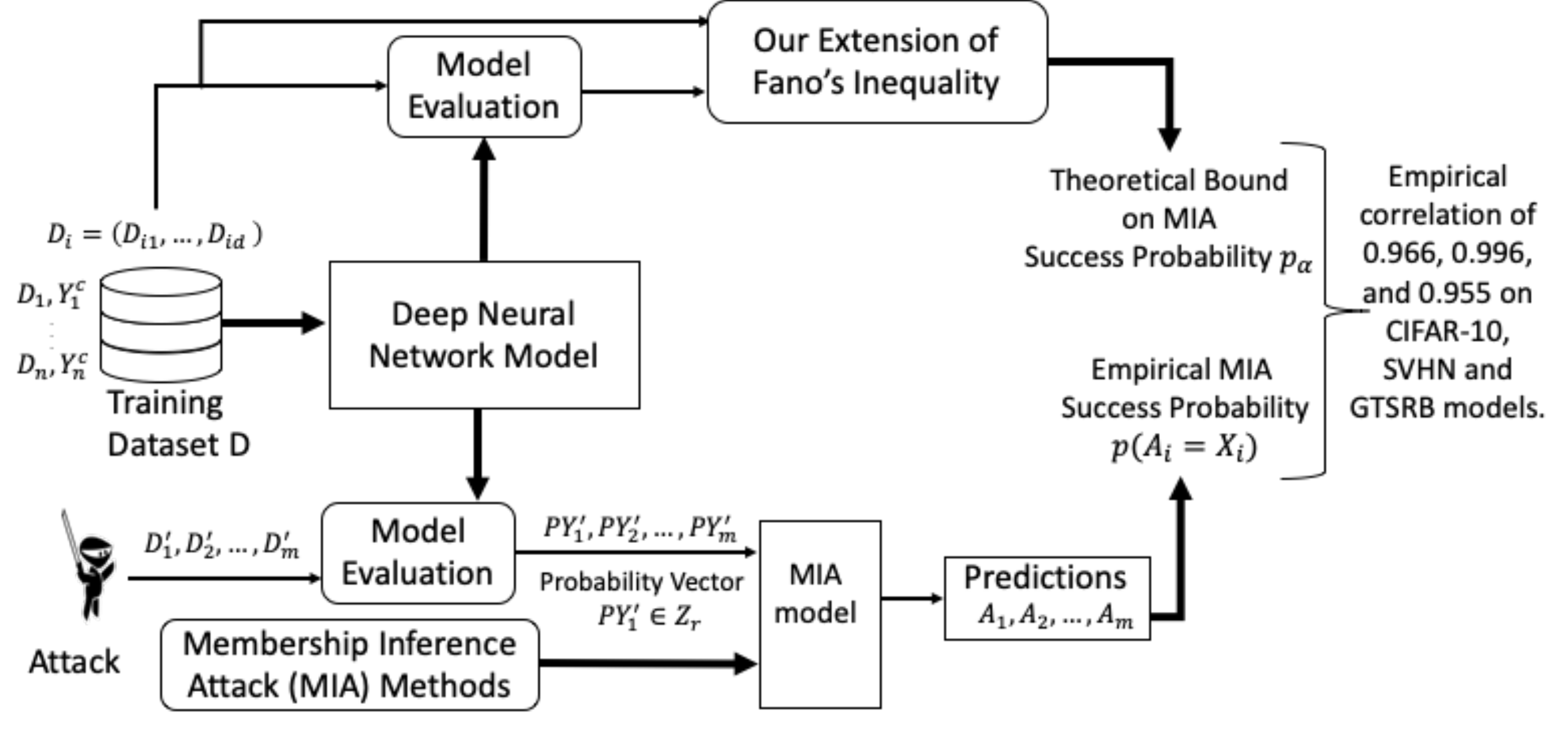} %
   \caption{The  training of a DNN model 
   $\nn{}$ uses the dataset $\inp{}$ with $n$ $d$-dimensional inputs 
   $\inp{i}$ and corresponding labels $\out{i}^c$. 
   An MIA attack method relies on feeding $m$ inputs $\inp{i}'$ 
   from  $ \ds{} \supseteq \inp{}$ to the 
   trained DNN model $\nn{}$ to obtain its probabilistic 
   predictions/activations. This 
   is, in turn, fed to the MIA
   model which makes a prediction $\ai{i}$ of whether 
   the data input $\inp{i}'$ is present in $\inp{}$. 
   The ground truth of whether $\inp{i}'$ is present in $\inp{}$ is denoted by $\gti{i}$. The output/activations are denoted by $Y$. The empirical MIA success 
   probability is computed from the predictions of the MIA model on whether the attacker-provided input data $\inp{i}'$  belong to the training set  $\inp{}$. 
   Our derived bound on the attack success probability is computed by 
   estimating the mutual information between the input 
   and the activation and gradient output of the top layers the DNN model $\nn{}$. Very high correlation demonstrates practical utility of our theoretical bound.}
   \label{fig:problem_def}
\end{figure*}

\section{Result Summary}

\textbf{MIA Attacker Model:} We consider an adversary mounting a membership inference attack against a DNN model  $\nn{}$ where the adversary can have black-box~\cite{ssss17} or white-box~\cite{nasr2019comprehensive} access to
the target DNN model. It can issue arbitrary queries $\inp{i}'$ and
retrieve the model's prediction $\out{i}$.
$\ds{} \supseteq \inp{}$ is the population from which the training dataset $\inp{}$ is drawn. The adversary can obtain the model output $Y$, which could be the softmax output in a black-box setting and include the activation and gradient output of top layers in a white-box environment. The adversary can access a set
of input data that are drawn independently from that population.
The attacker's inputs $\inp{}'$ might contain only elements of interest to the attacker for which it wants to infer whether these were used in training the model $\nn{}$. 
The adversary has no other information about whether these input data are present in the training set.

\noindent\textbf{Susceptibility of a model to MIA Attack:} Given a specific input $\inp{i}$ from a data set $\ds{}$ and a neural network $\nn{}$ learned from  the training data  $\inp{} \subseteq \ds{}$, an MIA attack $\mia{}$  determines whether $\inp{i} \in \inp{}$, i.e. the input $\inp{i}$ is present in the training data set $\inp{}$. Let $\gt\ = (\gti{0}, \dots, \gti{m})$ be a random variable that indicates the ground truth whether the attack inputs are present in the training data set $\inp{}$. Here, $\gti{i}=1$ if the training data set $\inp{}$ contains the corresponding data $\inp{i}$; otherwise $\gti{i}=0$.
  $\ai{} = (\ai{0}, \dots, \ai{m})$ denotes a random variable describing whether an MIA algorithm labels the data $\inp{i}$ as being present in the training set for model $\nn$. $\ai{i}=1$ if the MIA algorithm predicts that the input $\inp{i}$ has been used for training; otherwise, $\ai{i}=0$. In this paper, we seek to answer the following questions: Can we establish a theoretical lower bound on the robustness of a DNN model against MIA attacks by analyzing the mutual information of its inputs and outputs?

\noindent \textbf{Key Observation: } As shown in Section~\ref{sec:exp}, the observed correlations between  the MIA success probability and the mutual information metric are $0.966$, $0.996$, and $0.955$ for CIFAR-10, GTSRB, and SVHN data sets.
The fact that 
these correlations are close to unity    suggests that we can compute the mutual information between input and output/activations of a model to estimate its susceptibility to MIA attacks.

\noindent \textbf{Extension of Fano's Inequality:} Given a DNN model $\nn{}$, the success probability $p_{\alpha}$ of a membership inference attack algorithm that considers all inputs from a data set $\ds{}$ %
making more than $\alpha$ prediction errors is
$$\pa  \geq   \frac{\en{\inp{}} - \mi{\inp{} ; \out{} } - 1 - \log{\left( {\scriptstyle \binom{|\ds{}|}{0} + \dots \binom{|\ds{}|}{\alpha} } \right)}}{|\ds{}| - \log{\left( {\scriptstyle \binom{|\ds{}|}{0} + \dots \binom{|\ds{}|}{\alpha} } \right)}}$$

Here, $\en{\inp{}}$ is the entropy of the training data set $\inp{}$, $|\ds|$ denotes the size of the total data set available to the MIA algorithm, and $\mi{\inp{};\out{}}$ denotes the mutual information between the training data $\inp{}$ and the neural network outputs/activations $\out{}$.
Since typically DNNs are deterministic functions, we add small noise to the DNN weights to compute this mutual information~\cite{achille2019information}. Only the mutual information term $\mi{\inp{};\out{}}$ depends on the DNN model, and hence, we can compute $\mi{\inp{};\out{}}$ to determine the robustness of the model - the higher the $\mi{\inp{};\out{}}$, the lower the robustness to MIA attacks. %

%% file: sectheorem.tex
\section{Theoretical Bounds on MIA Success using Extension of Fano's Inequality}

The supervised training of a DNN model $\nn{}$ uses the dataset $\inp{}$ with inputs 
   $\inp{i}$ and corresponding labels $\out{i}^c$. The probabilistic output of 
   the  DNN model on the data set $\inp{}$ is denoted by $\out{}$. A MIA method relies on feeding inputs $\inp{i}'$ 
   from a data set $ \ds{} \supseteq \inp{}$ to the 
   trained DNN model to obtain its probabilistic 
   prediction (softmax layer output) $\out{i}$. This 
   is, in turn, fed to the MIA
   model which makes a prediction $\ai{i}$ of whether 
   the data input $\inp{i}'$ is present in $\inp{}$. 
   The ground truth of whether $\inp{i}'$ is present in $\inp{}$ is denoted by $\gti{i}$.
 The error $\error$ of the attack model is given by $\error{} =  \sum_i \mathds{1}(X_i \neq A_i)$, where $\mathds{1}$ is the indicator function. For a threshold $\alpha$, we can define an indicator random variable $\ea$ that has the value $1$ when $\error{} > \alpha$  and $0$ otherwise.  
 We use the notation $\pa = Pr(\ea=1)$ to denote the probability of the event $\ea$. 

\subsection{Fano's Inequality}
We briefly recall a classical result from information theory, Fano's inequality, that  establishes a relationship between the average information lost in a noisy channel and the probability of the categorization error~\cite{fano1961transmission}. 

Let $X_F$ represent the input to a noisy channel being analyzed by Fano's inequality, and $Y_F$ represent the corresponding output on this channel. Further, let $P(x_F,y_F)$ denote the joint probability of the input and the output for this noisy channel.

Suppose the random variable 
$e_F$ represents the occurrence of an error in the noisy channel, i.e., the approximate recovered signal  $\tilde{X}_F=f(Y_F)$ is not the same as the input signal. Formally, $e_F$ corresponds to the event 
$X_F \neq {\tilde  {X_F}}$. We denote the support of the random variable $X_F$ by the notation $\mathcal  {X}_F$.

Fano's inequality establishes a fundamental information-theoretic relationship between the conditional information $\en{ X_F|Y_F}$ and the probability of error $P(e_F)$ in a noisy channel:
\begin{align}
   \en{ X_F|Y_F} \; \leq \;  \en{e_F} + P(e_F) \log(|{\mathcal  {X}_F}|-1)  
\end{align}

\noindent Our mathematical results in this paper are an extension of Fano's inequality that relate the probability of error of a membership inference attack with the mutual information between the inputs and outputs/activations of a neural network.

\subsection{Extension of Fano's Inequality to MIA Success}

We theoretically establish a relationship between the probability of a MIA model making $\alpha$ prediction errors on a neural network $\nn{}$ and the mutual information $\mi{\inp{};\out{}}$ between the inputs $\inp{}$ and the outputs/activations $\out{}$ of the neural network $\nn{}$. Our proof procedure first establishes two lemmas on the conditional entropy $\en{\ea{}, \gti{} | \ai{}}$, and then uses these results to prove a theorem relating MIA prediction errors with the mutual information $\mi{\inp{};\out{}}$. Our proof of the bound on $p_\alpha$ is applicable to any classifier with input $\inp{}$ and output $\out{}$, not just to a neural network.

\begin{lemma}
\label{lemma1}
\begin{align*}
\en{ \ea, \gti{} | \ai{} } = \en{ \gti{} | \ai{} } 
\end{align*}
\end{lemma}
\begin{proof}
Since the error $\ea$ is deterministically known given $\gti{}$ and $\ai{}$, the entropy $\en{ \ea{} | \gti{} , \ai{} } = 0$. We can evaluate $\en{ \ea, \gti{} | \ai{} }$ using the chain rule of conditional entropy. 
\begin{align}
\en{ \ea{}, \gti{} | \ai{} } & = \en{ \gti{} | \ai{} } + \en{ \ea{} | \gti{} , \ai{} } \\
\nonumber & = \en{ \gti{} | \ai{} } + 0 \\
& = \en{ \gti{} | \ai{} }  
\label{EqnAppendix:H1}
\end{align}
\end{proof}
\begin{lemma}
\label{lemma2}
\begin{align*}
\en{ \ea{}, \gti{} | \ai{} }  \leq  1 +  (1-\pa{}) \binomialseries{}  + \pa{} |\ds{}|
\end{align*}
\end{lemma}

\begin{proof}
We perform an expansion for $\en{ \ea, \gti{} | \ai{} }$ using the chain rule of conditional entropy:
\begin{align}
\en{ \ea{}, \gti{} | \ai{} } & = \en{ \ea{} | \ai{} } + \en{ \gti{} | \ea{} , \ai{} }  \label{EqnAppendix:H2}
\end{align}
\noindent  Now, we know that $ \en{ \ea | \ai{} } \leq  \en{ \ea }$ as conditional entropy is no more than an unconditional entropy. Further, since $\ea$ is a binary valued random variable, $\en{\ea} \leq 1$ by the definition of entropy. Thus, we can write Eqn.~\ref{EqnAppendix:H2} as follows:
\begin{align}
\en{\ea, \gti{} | \ai{} } & \leq  1 + \en{ \gti{} | \ea , \ai{} }  
\label{EqnAppendix:H2a}
\end{align}
\noindent We can expand the second term $\en{ \gti{} | \ea , \ai{} }$ by splitting $\ea$ into two cases i.e. $\ea=0$ and $\ea=1$:
\begin{align}
\en{ \gti{} | \ea , \ai{} } & =  Pr(\ea=0) \en{ \gti{} | \ea =0 , \ai{} } \nonumber \\ & + Pr(\ea=1) \en{ \gti{} | \ea =1 , \ai{} }  
\label{EqnAppendix:H3}
\end{align}
\noindent We can simplify the above expression by obtaining bounds on the quantity $\en{ \gti{} | \ea =0 , \ai{} } $.  If $\ea=0$, the random variable $\gti{}$ can only differ from the random variable $\ai{}$ in at most $\alpha$ positions. 
Thus, given a particular value of the random variable $\ai{}$, the random variable $\gti{}$ can only take at most ${ { \binom{|\ds{}|}{0} + \dots + \binom{|\ds{}|}{\alpha} } }$ 
= $V(\alpha)$ values.  The highest entropy is achieved when all these values are equally likely i.e. $\en{ \gti{} | \ea =0 , \ai{} }  \leq - \sum_{j=1}^{V(\alpha)} \frac{1}{V(\alpha)} \log \frac{1}{V(\alpha)} =  - \log \frac{1}{V(\alpha)}  \sum_{j=1}^{V(\alpha)} \frac{1}{V(\alpha)}    = - \log \frac{1}{V(\alpha)} = \log{V(\alpha)}$.
Hence, Eqn.~\ref{EqnAppendix:H3} can be rewritten as:
\begin{align}
\en{ \gti{} | \ea{} , \ai{} } & \leq  (1-\pa{}) \binomialseries{} \nonumber \\& + \pa{} \; \en{ \gti{} | \ea{} =1 , \ai{} }  
\label{EqnAppendix:H3a}
\end{align}
\noindent In the above equation, we have used $\pa$ as a shorthand to represent the probability $Pr(\ea=1)$. Since $X$ can take at most $2^{|\ds|}$ different values, the term $\en{ \gti{} | \ea =1 , \ai{} }$ on the right can be upper bounded by $\log{2^{|\ds|}}$ = $|\ds|$ using the definition of entropy. Thus, Eqn.~\ref{EqnAppendix:H3a} can be simplified as:
\begin{align}
\en{ \gti{} | \ea{} , \ai{} } & \leq  (1-\pa{}) \binomialseries{} + \pa{} |\ds{}|
\label{EqnAppendix:H3b}
\end{align}
\noindent Putting together equations~\ref{EqnAppendix:H2a} and~\ref{EqnAppendix:H3b}, we get the following:
\begin{align}
\en{ \ea, \gti{} | \ai{} } & \leq  1 + (1-\pa) \log{\left( {\scriptstyle \binom{|\ds{}|}{0} + \dots \binom{|\ds{}|}{\alpha} } \right)}  + \pa |\ds{}|
\label{EqnAppendix:H4}
\end{align}
\end{proof}

\begin{theorem}
\label{mainthem}
Given a neural network $\nn{}$ and a MIA model that  considers all inputs from a data set $\ds{}$ and only observes the outputs/activations $\out{}$ of the neural network $\nn{}$, the probability of such a MIA model making more than $\alpha$ prediction errors is
$$\pa  \geq   \frac{\en{\inp{}} - \mi{\inp{} ; \out{} } - 1 - \log{\left( {\scriptstyle \binom{|\ds{}|}{0} + \dots \binom{|\ds{}|}{\alpha} } \right)}}{|\ds{}| - \log{\left( {\scriptstyle \binom{|\ds{}|}{0} + \dots \binom{|\ds{}|}{\alpha} } \right)}}$$

\noindent Here, $\en{\inp{}}$ is the entropy of the training data set $\inp{}$, $|\ds|$ denotes the size of the total data set available to the MIA, and $\mi{\inp{};\out{}}$ denotes the mutual information between the training data $\inp{}$ and the  outputs/activations $\out{}$ of the neural network.

\end{theorem}
\begin{proof}

\noindent Putting together the results from Lemma~\ref{lemma1} and Lemma~\ref{lemma2}, we obtain the following:
\begin{align}
& \en{ \gti{} | \ai{} }  \leq  1 + (1-\pa{}) \binomialseries{}  + \pa |\ds{}| \nonumber \\ 
\label{Eqn:pa4x} \implies & \pa  \geq   \frac{\en{ \gti{} | \ai{} } - 1 - \binomialseries{}}{|\ds{}| - \binomialseries{}}
\end{align}
Note that $\gti{}$ is determined given the training data $\inp{}$ used to train the neural network $\nn{}$; hence, $\en{\gti{}|\inp{},\ai{}}=0$. Thus, using the chain rule of conditional entropy, we get $\en{\inp{}, \gti{} | \ai{}} = \en{\inp{} | \ai{}} + \en{\gti{}|\inp{},\ai{}} = \en{\inp{} | \ai{}} + 0 = \en{\inp{} | \ai{}}$
Also, repeating the chain rule of conditional entropy, we  get $\en{\inp{}, \gti{}| \ai{}} = \en{\gti{}|\ai{}} + \en{\inp{}|\gti{},\ai{}}$. Combining these two results, we obtain the following: 
$\en{\gti{}|\ai{}} = \en{\inp{} | \ai{}} - \en{\inp{} | \gti{}, \ai{}}$.
Putting this together with Eqn.~\ref{Eqn:pa4x}, we obtain the following:

\resizebox{1.01  \linewidth}{!}{
  \begin{minipage}{1.0\linewidth}
 \begin{align}
\pa  & \geq   \frac{\en{ \inp{} | \ai{}} - \en{ \inp{} | \gti{}, \ai{} } - 1 - \binomialseries{} }{|\ds{}| - \binomialseries{}}  \nonumber \\
&  \hspace{-0.4cm} \geq   \frac{\en{\inp{}} - \mi{\inp{} ; \ai{} } - \en{ \inp{} | \gti{}, \ai{} }- 1 - \binomialseries{}}{|\ds{}| - \log{\left( {\scriptstyle \binom{|\ds{}|}{0} + \dots \binom{|\ds{}|}{\alpha} } \right)}}  \nonumber\\
& \hspace{3.3cm} \text{ as } \mi{\inp{};\ai{}} = \en{\inp{}} - \en{\inp{}|\ai{}} \nonumber \\
& \hspace{-0.4cm} \geq   \frac{\en{\inp{}} - \mi{\inp{} ; \ai{} } - 1 - \binomialseries{}}{|\ds{}| - \binomialseries{}} \nonumber \\
& \label{Eqn:H5a} \hspace{3.3cm} \text{since, } \en{\inp{} | \gti{}, \ai{} } = 0
\end{align}
  \end{minipage}
}
Also, since $\out{}$ is obtained from $\inp{}$ by using the neural network $\nn{}$, and the adversarial prediction $\ai{}$ is obtained from the neural network response $\out{}$, the data processing inequality implies that $\mi{\inp{};\ai{}} \leq \mi{\inp{};\out{}}$. Applying these results to Eqn.~\ref{Eqn:H5a}, we get the following:
\begin{align}
& \pa  \geq   \frac{\en{\inp{}} - \mi{\inp{} ; \out{} } - 1 - \binomialseries{}}{|\ds{}| - \binomialseries{}} 
\label{Eqn:H5b}
\end{align}
\end{proof}

The training of a neural network does not influence the entropy of the training data set $\en{\inp{}}$ or the size of the complete data set $\ds{}$ used by the membership inference attack.
Our analysis shows that the probability of a membership inference attack making more than $\alpha$ prediction errors is dependent on the  mutual information $\mi{\inp{};\out{}}$ between the inputs and the outputs/activations of a neural network. 
Thus, the mutual information between the inputs and the outputs/activations of a neural network can be used to characterize its susceptibility to membership inference attacks.  

\begin{example}[Theorem~\ref{mainthem} with $\mi{\inp{};\out{}}=0$,  $\alpha=c$ where $c$ is a constant such that $c<<|\ds{}|$, and $\en{\inp{}}=|\ds{}|$]
\label{example1}
Consider an untrained   neural network such that the mutual information between its input $\inp{}$ and its output $\out{}$ is zero. Further, assume that $\en{\inp{}} = |\ds{}|$.  
Then, Theorem~\ref{mainthem} states that the probability $\pa$ of a membership inference attack making more than $c$ prediction errors is:  
\begin{align*}
    \pa & \geq   \frac{\en{\inp{}} - \mi{\inp{} ; \out{} } - 1 - \log{\left( {\scriptstyle \binom{|\ds{}|}{0}  + \dots + \binom{|\ds{}|}{c} }  \right)} }{|\ds{}| - \log{\left( {\scriptstyle \binom{|\ds{}|}{0}  + \dots + \binom{|\ds{}|}{c} }  \right)} }\\
& \geq   \frac{|\ds{}| - 1 - \log{\left( {\scriptstyle \binom{|\ds{}|}{0}  + \dots + \binom{|\ds{}|}{c} }  \right)} }{|\ds{}| - \log{\left( {\scriptstyle \binom{|\ds{}|}{0}  + \dots + \binom{|\ds{}|}{c} }  \right)} } \\ &
\hspace{2.5cm} \textrm{Since, } \mi{\inp{}; \out{}}=0  \textrm{ and } \en{\inp{}} = |\ds{}|\\
& \geq  1 -  \frac{1}{|\ds{}| - \log{\left( {\scriptstyle \binom{|\ds{}|}{0}  + \dots + \binom{|\ds{}|}{c} }  \right)}} 
\end{align*}

As the data set becomes large i.e. $|\ds{}| \rightarrow \infty$, $p_\alpha \rightarrow 1$ for $\alpha=c << |\ds{}|$ i.e. the membership inference attack will almost surely make at least $c$ prediction errors if   $\mi{\inp{};\out{}}=0$ and $\en{\inp{}}=|\ds{}|$.   
\end{example}

Example~\ref{example1} shows how the probability bound established by Theorem~\ref{mainthem} ties  with our intuition in a specific setting of a poorly trained neural network with $\mi{\inp{};\out{}}=0$. Now, we look at another example of a neural network where $\mi{\inp{},\out{}}=|\ds{}|/c$ for some constant $c>1$.
\begin{example}[Theorem~\ref{mainthem} with $\mi{\inp{};\out{}}=|\ds{}|/c$ for some constant $c>1$, $\alpha=0$, and $\en{\inp{}}=|\ds{}|$]
\label{example2}
Consider a neural network whose mutual information is given by $\mi{\inp{};\out{}}=|\ds{}|/c$. 
Applying Theorem~\ref{mainthem}, the probability of making one or more prediction errors is:  
\begin{align*}
    \pa & \geq   \frac{\en{\inp{}} - \mi{\inp{} ; \out{} } - 1 - \log{\left( {\scriptstyle \binom{|\ds{}|}{0} } \right) }}{|\ds{}| - \log{\left( {\scriptstyle \binom{|\ds{}|}{0} }  \right)} }\\
& \geq   \frac{|\ds{}| - \frac{|\ds{}|}{c} - 1 }{|\ds{}|}  \hspace{0.15cm} \textrm{Since, } \mi{\inp{};\out{}}=\frac{|\ds{}|}{c} \textrm{ and }  \en{\inp{}} = |\ds{}| \\
& \geq  1 -  \frac{1}{c} - \frac{1}{|\ds{}|}
\end{align*}
Thus, according to Theorem~\ref{mainthem}, a membership inference attack may make at least one prediction error with probability $1-\frac{1}{c}$ as $|\ds{}| \rightarrow \infty$  .
\end{example}

%% file: secalgo.tex
\subsubsection{Measuring Mutual Information:} Entropy of any $d$ dimensional random variable $x$ can be computed using a non-parametric estimator~\cite{gao2015efficient} based on $k$-nearest-neighbors (kNN)  with a correction applied for  the local non-uniformity of the underlying joint distribution of the $d$ features. A simple kNN based estimator for entropy from  samples $x^1, x^2,  \ldots, x^n$ is:
    $\en{x} = -\frac{1}{n} \sum_{1}^{n}{
    \log {p_k(x^i)}}$
where the probability density is given by $p_k(x^i) = \frac{k}{n-1} \frac{\Gamma(d/2+1)}{\pi^{d/2}} r_k (x^i)^{-d}$. Here, $r_k(x^i)$ is the distance between $x_i$ and its $k^{th}$ nearest neighbor in the data set.
This can be used to compute the entropy of training data 
$\en{\inp{}}$, $\en{\out{}}$, and $\en{\inp{},\out{}}$.
The empirical estimation of the mutual information between the training inputs and outputs/activations of a DNN model $\mi{\inp{};\out{}}$ is obtained as 
$\mi{\inp{};\out{}} = \en{\inp{}} + \en{\out{}} - \en{\inp{},\out{}}$.

%% file: secrelated.tex
\section{Related Work}

\begin{figure*}[h!]
    \centering
\begin{tikzpicture}[scale=0.8]
\begin{axis}[
    ybar,
    title={\Large CIFAR-10 \normalsize},
    ylabel={\large Attack Success Probability \normalsize},
    symbolic x coords={MI=2.01,MI=1.52,MI=1.24,MI=1.17},
	xlabel=Models,
	xtick=data,
	legend style={at={(0.5,-0.1)},
	anchor=north,legend columns=4},
    enlarge x limits={true,abs value=1cm},
    x = 2cm
]

 \addplot[
    fill=black
    ]
    coordinates {
      (MI=2.01, 0.820943442825389 )
       (MI=1.52, 0.6811014356319134 )
       (MI=1.24, 0.6534785507890642 )
       (MI=1.17, 0.5945189037807561 )
        };
    \addlegendentry{Attack  1 }

 \addplot[
    fill=brown
    ]
    coordinates {
      (MI=2.01, 0.8174857989627068 )
       (MI=1.52, 0.7018121911037891 )
       (MI=1.24, 0.6725939097577239 )
       (MI=1.17, 0.6289257851570315 )
        };
    \addlegendentry{Attack 2  }

 \addplot[
    fill=blue
    ]
    coordinates {
      (MI=2.01, 0.8327982217831563 )
       (MI=1.52, 0.7145210637797129 )
       (MI=1.24, 0.6957101578128473 )
       (MI=1.17, 0.6517303460692139  )
        };
    \addlegendentry{Attack  3 }
\end{axis}
\end{tikzpicture}
\; \; \; \; 
\begin{tikzpicture}[scale=0.8]
\begin{axis}[
    ybar,
    title={\Large GTSRB \normalsize},
    ylabel={\large Attack Success Probability \normalsize},
    symbolic x coords={MI=2.29,MI=1.62,MI=1.32,MI=1.21},
	xlabel=Models,
	xtick=data,
	legend style={at={(0.5,-0.1)},
	anchor=north,legend columns=4},
    enlarge x limits={true,abs value=1cm},
    x = 2cm
]

 \addplot[
    fill=black
    ]
    coordinates {
      (MI=2.29, 0.9397470865360774 )
       (MI=1.62, 0.6881591562799616 )
       (MI=1.32, 0.6132075471698113 )
       (MI=1.21, 0.5747763101832126 )
        };
    \addlegendentry{Attack  1 }

 \addplot[
    fill=brown
    ]
    coordinates {
      (MI=2.29, 0.9843788742871312 )
       (MI=1.62, 0.6421380632790029 )
       (MI=1.32, 0.6318453750575241 )
       (MI=1.21, 0.5125692373242438 )
        };
    \addlegendentry{Attack 2  }

 \addplot[
    fill=blue
    ]
    coordinates {
      (MI=2.29, 0.9833870567815522 )
       (MI=1.62, 0.7209971236816874 )
       (MI=1.32, 0.5883571099861942 )
       (MI=1.21, 0.4629314017895185 )
        };
    \addlegendentry{Attack  3 }
\end{axis}
\end{tikzpicture}

\medskip 

\begin{tikzpicture}[scale=0.8]
\begin{axis}[
    ybar,
    ymin=0,
    title={\Large SVHN \normalsize},
    ylabel={\large Attack Success Probability \normalsize},
    symbolic x coords={MI=1.09,MI=0.88,MI=0.51,MI=0.40},
	xlabel=Models,
	xtick=data,
	legend style={at={(0.5,-0.1)},
	anchor=north,legend columns=4},
    enlarge x limits={true,abs value=1cm},
    x = 2cm
]

 \addplot[
    fill=black
    ]
    coordinates {
      (MI=1.09, 0.9774066797642437 )
       (MI=0.88, 0.8531802965088474 )
       (MI=0.51, 0.5497780596068484 )
       (MI=0.40, 0.4602058319039451 )
        };
    \addlegendentry{Attack  1 }

 \addplot[
    fill=brown
    ]
    coordinates {
      (MI=1.09, 0.9769155206286837 )
       (MI=0.88, 0.9017216642754663 )
       (MI=0.51, 0.6026210103572184 )
       (MI=0.40, 0.4716981132075472 )
        };
    \addlegendentry{Attack 2  }

 \addplot[
    fill=blue
    ]
    coordinates {
      (MI=1.09, 0.981827111984283 )
       (MI=0.88, 0.8981348637015782 )
       (MI=0.51, 0.7059818220249419 )
       (MI=0.40, 0.47787307032590054 )
        };
    \addlegendentry{Attack  3 }

\end{axis}
\end{tikzpicture}
\; \; \; \; \; \; \hspace{1cm}
\begin{tikzpicture}{scale=0.8}
\node (tbl) {
\begin{tabular}{@{}cc@{}}
\toprule 
Data Set    &  Correlation\\
            & MI \& Attack Probability\\
\midrule
 CIFAR-10   & 0.966\\
 GTSRB   & 0.996\\
 SVHN   & 0.955\\
 \bottomrule
\vspace{0.75cm}
\end{tabular}
};

\end{tikzpicture}
    \caption{Mutual information between the inputs and the output layers of a neural network correlates strongly with the success probability of membership inference attack models. The Pearson correlations between mutual information and success probability of a contemporary MIA attack~\cite{shokri2015privacy}  are $0.966$, $0.996$ and $0.955$ for neural networks trained on the CIFAR-10, GTSRB and SVHN data sets, respectively. Our evaluation considers three different variants of the MIA attack.}
    \label{fig:miprob}
\end{figure*}
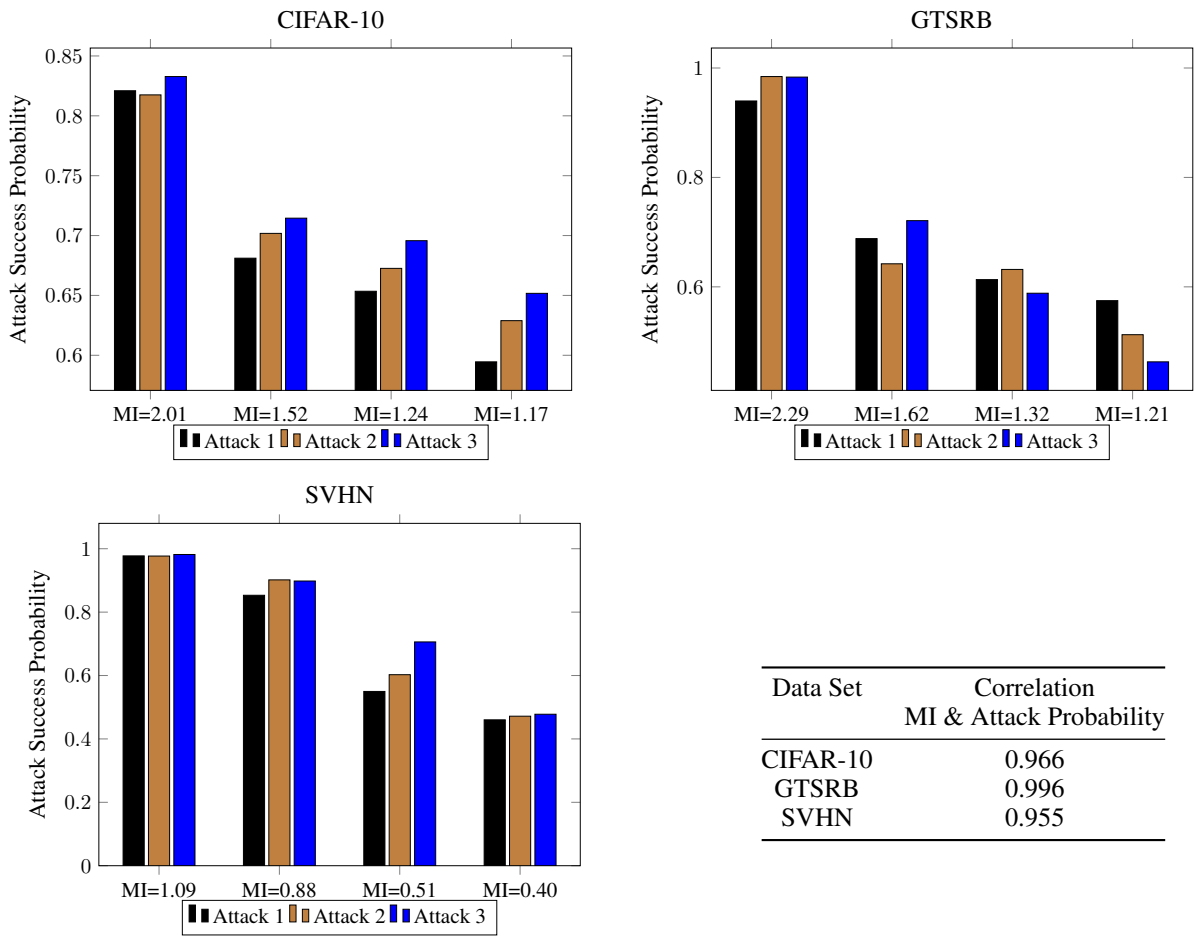

We  survey related work in membership inference attacks and discuss privacy preserving approaches to machine learning. We sketch the relationship between regularization, mutual information and generalization in deep neural networks. 

\subsection{Membership Inference Attacks}
A membership inference attack on neural networks essentially generalizes the well-studied problem of identifying if a specific data record is present in a 
data set given some statistic about this data set~\citep{ssss17,nasr2019comprehensive, jacobs2009new, sankararaman2009genomic}. This is a severe privacy concern. For example, membership in the training data set of a model associated with an addiction or disease can reveal otherwise private information about the patient~\citep{liu2019socinf, ptc17}. A number of MIA methods have been proposed recently in literature. One approach~\citep{ssss17} trains a number of shadow models independently using a subset of the training dataset. 
The final attacker model learns from all these shadow models, and can then predict if a data element was in or out of the target model's training data. Another training-time attack is based on augmenting the training data with additional synthetic inputs whose labels encode information that the model needs to leak~\citep{song2017machine}. No other component of the entire training pipeline is perturbed.  Yet another approach~\cite{melis2019exploiting} exploits the fact that deep neural networks construct multiple internal representations of all kinds of features related to the input data, including those irrelevant to the current task.
These attacks have also been extended to collaborative and federated settings~\cite{melis2019exploiting}.
Robust learning techniques to defend against adversarial attacks have been shown to increase susceptibility to MIA  attacks~\cite{song2019privacy}. Finally, these attacks have also been shown to be largely transferable~\cite{truex2018towards}. These observations further underline the need for addressing MIA  attacks.

\subsection{Privacy Preserving Machine Learning}
Differential privacy is used for privacy-preserving statistical analysis over sensitive data where the privacy and utility trade-off is controlled by a privacy budget parameter. 
Differential privacy can provide formal guarantees that the model trained on a given dataset will produce statistically similar predictions as a model trained on a different dataset that differs by exactly one instance~\cite{dwork2014algorithmic}. 
Differential training privacy has been proposed as a way to measure model susceptibility by computing this worst-case difference among all training data points~\cite{long2017towards}.
These are particularly useful for simple convex machine learning algorithms~\cite{chaudhuri2011differentially,zhang2016differential,jayaraman2018distributed}. But differential private deep learning often requires a large privacy budget~\cite{shokri2015privacy} with ongoing efforts to reduce it~\cite{abadi2016deep,hynes2018efficient}. Differential privacy methods can provide worst-case bounds on the privacy loss, but these do not provide an understanding of privacy attacks in practice.  Membership and attribute inference attacks, on the other hand, provide an empirical lower bound on the privacy loss of training data. 
The relationship between the standard worst-case definition of differential privacy and the average-case mutual-information notion is an active area of study in the security and privacy literature~\cite{cuff2016differential,wang2016relation}. 
 Further, MIA attacks are a restricted form of privacy attacks that do not aim at discovering the training data but only detecting the presence of a given data in the training set. 
In contrast to the differential privacy bounds, we focus entirely on MIA attacks and formulate an information theoretic bound on the probability of such an attack being successful instead of characterizing worst-case privacy leakage. This allows a scalable and practical approach to measure and regulate the average-case susceptibility of DNN models to existing MIA attacks.
In order to make DNN models more robust to privacy attacks, there are broadly two classes of techniques.
The first relies on adding noise directly to the training inputs~\cite{zhang2018privacy}, or to the stochastic gradient
descent~\cite{abadi2016deep} to control the affects of the training data on the model parameters. The second class uses an  aggregation of teacher ensembles~\cite{dwork2018privacy,papernot2018scalable,ptc17},
where privacy is enforced by training each teacher on a separate subset of training data, and relying on the noisy aggregation of the teachers' responses.

\subsection{Generalization and Memorization in DNNs}
A desirable property of any model is having low generalization error, that is, good performance on unseen examples from the population. 
The connection between overfitting and membership inference attacks has also been investigated~\cite{yeom2018privacy}.
Regularization techniques aimed at controlling model complexity have been traditionally used to reduce overfitting and improve generalization. But recent work has demonstrated that these regularization techniques do not reduce the susceptibility to MIA attack~\cite{long2018understanding}.  
In contrast, we use mutual information to characterize susceptibility of DNNs to MIA attacks. 
One explanation of generalization in deep learning states that training initially increases the mutual information between the input and the output of the model, and then decreases the mutual information removing relations irrelevant to the task and improving generalization~\cite{shwartz2017opening}. 
A related effort focuses on the ability of a deep learning model to unintentionally memorize unique or rare sequences in the training data~\citep{carlini2018secret}, and uses it to measure the model's propensity for leaking training data. 
Prior work has shown that deep learning models can be trained to perfectly fit completely random data~\cite{zhang2016understanding} which indicates high memorization capacity of DNNs. Hence, MIA attacks are not an oddity of a particular learning technique or model, but a result of the widely observed memorization in deep learning models. Our approach of characterizing MIA susceptibility of models to these attacks to mutual information is, thus, a first step in a promising direction that connects privacy and generalization of DNNs.

%% file: secexp.tex
\newpage

\section{Experimental Results}
\label{sec:exp}

\begin{figure*}[htbp]
    \centering
    \includegraphics[width=7.25cm]{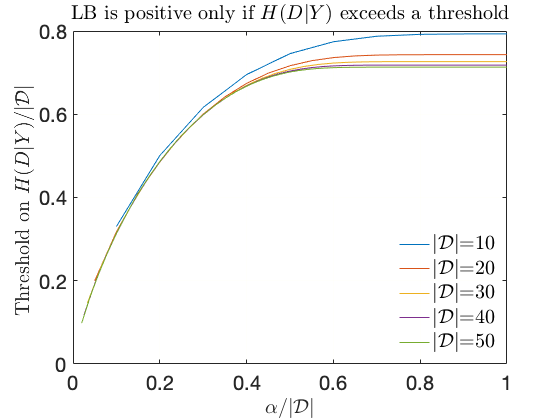}
    \includegraphics[width=7.25cm]{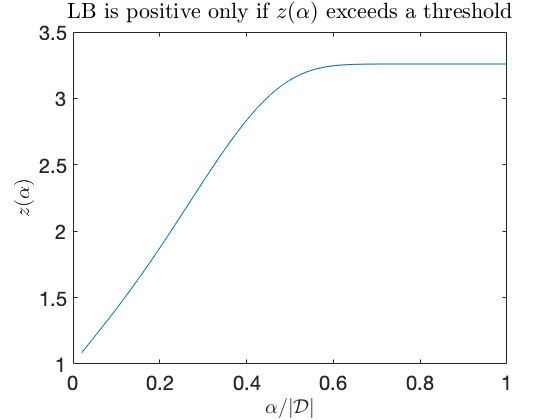}
    \caption{The approximate lower bound (LB) on $p_{\alpha}$ is positive only when $\en{\inp{}|\out{}}$ exceeds a threshold (left) and $z(\alpha) = \binomialseries{}/|\ds{}| $ exceeds another threshold (right).}
    \label{fig:HDY}
\end{figure*}

Our experiments are performed on a system with 128GB RAM, a 16-core AMD processor, and 2 NVIDIA RTX 2080 Ti GPUs running Ubuntu 20.04.
Three popular data sets are used for our investigations: (i) CIFAR-10~\cite{krizhevsky2014cifar} (ii) SVHN~\cite{37648} and (iii) GTSB~\cite{houben2013detection}. 
In our experimental evaluation, we investigate {\it {whether we can use mutual information between the input and output of a DNN model to estimate the success probability of  MIA attacks on the model.}}

\subsubsection{CIFAR-10: } We study 4 DNN models for the CIFAR-10 data set with mutual information decreasing from $2.01$ to $1.17$ nats. Using three different variants of a contemporary membership inference attack~\cite{nasr2019comprehensive} with $3$, $5$ and $7$ shadow models, the probability of attacks decreases from $0.82$ to $0.59$, $0.82$ to $0.62$, and $0.83$ to $0.65$, for the three attacks respectively. A decrease in mutual information is coupled with a decrease in the success probability of the MIA model. The Pearson correlation between the mutual information and the attack probability for CIFAR-10 is $0.966$.

\subsubsection{GTSRB: } The GTSRB data is also used to train four different neural network models with mutual information decreasing from $2.29$ to $1.21$. As shown in Fig.~\ref{fig:miprob}, the success probability of the most powerful MIA model falls from $0.98$ to $0.46$.

\subsubsection{SVHN: } Similar reduction in the success of MIA models is observed on the SVHN data set. As the mutual information falls from $1.09$ to $0.40$, the probability of success of the most successful MIA model falls from $0.98$ to $0.47$. 

    We find that the Pearson correlations between mutual information and success probability of a contemporary MIA attack~\cite{nasr2019comprehensive} are $0.966$, $0.996$ and $0.955$ for neural networks trained on the CIFAR-10, GTSRB and SVHN data sets, respectively. The strongly positive Pearson's correlation across data sets confirms our theoretical finding that mutual information is related to the success probability of MIA models.

\subsection{Broader Applicability of Our Lower Bound}
While Theorem~\ref{mainthem} enables a theoretical understanding of the relationship between mutual information $\mi{\inp{};\out{}}$, in this section, we investigate an orthogonal question: {\it {when does Theorem~\ref{mainthem} produce positive lower bounds on $p_{\alpha}$? }}

 Figure~\ref{fig:HDY} (left) shows a plot of a threshold on the ratio of the conditional entropy $\en{\inp{}|\out{}}$ and the size of the data set $|\ds{}|$ such that conditional entropy values higher than this approximate threshold are required for a positive lower bounds for $p_{\alpha}$ in Theorem~\ref{mainthem}. We can verify that the results agree with our intuition for various values of the ratio of the number of errors $\alpha$ to the size of the data set $|\ds{}|$. For example, if we are only interested in small number of errors $\alpha < \frac{|\ds{}|}{4}$, our lower bounds on $p_{\alpha}$ are positive when  $\en{\inp{}|\out{}} >  \frac{|\ds{}|}{2}$ i.e. the conditional entropy $\en{\inp{}|\out{}}$ is comparable to at least half the size of the data set $|\ds{}|$.

On the other hand, as the size of the data set increases and the number of errors becomes large e.g. $\alpha/|\ds{}|\approx 0.5$, the curves corresponding to the threshold show that the conditional entropy $\en{\inp{}|\out{}}$ needs to become as large as about $71\%$ of $|\ds{}|$ for our lower bound to produce a positive result. This again makes intuitive sense as the conditional entropy must be high in order for even the best membership inference attack to suffer a large number of errors.

The bound in Theorem~\ref{mainthem} can also be stated as $p_{\alpha} \geq 1 - \frac{1+1/|\ds{}|-c}{1-z(\alpha)}$, where $c=\frac{\en{\inp{}|\out{}}}{|\ds{}|}$ and $z(\alpha)=\frac{\binomialseries{}}{|\ds{}|}$. Figure~\ref{fig:HDY} (right) shows how the value of $z(\alpha)$ required for a positive lower bound changes with the ratio $\alpha/|\ds{}|$ in one setting. 

In summary, our lower bound on $p_{\alpha}$ is useful in a large non-degenerate regime where the conditional entropy $\en{\inp{}|{\out{}}}$ is not too low when compared to the size of the data set $|\mathcal{D}|$. If the conditional entropy $\en{\inp{}|{\out{}}}$ is too low, our bound is not positive and this ties well with our intuition that a good adversary can launch embarrassingly successful membership inference attacks in this setting.